\documentclass[11pt,a4paper]{article}
\usepackage[hyperref]{acl2021}
\usepackage{times}
\usepackage{latexsym}

\usepackage{microtype}

\usepackage{hyperref}
\usepackage{soul}
\usepackage{makecell}
\usepackage{caption}
\usepackage{subcaption}
\usepackage{multirow}
\usepackage{booktabs}
\usepackage{color}
\usepackage{amsmath}
\usepackage{amssymb}
\usepackage{mathtools}
\usepackage{amsthm}
\usepackage[ruled,vlined]{algorithm2e}
\newtheorem{theorem}{Theorem}
\newtheorem{lemma}{Lemma}
\newtheorem{definition}{Definition}
\usepackage{caption}
\usepackage{graphicx}

\aclfinalcopy 

\title{Knowledge Injection into Dialogue Generation\\via Language Models}

\author{Yi-Lin Tuan$^1$, Wei Wei$^2$, William Yang Wang$^1$ \\
  $^1$University of California, Santa Barbara, USA\\
  $^2$Google Research, Mountain View, USA\\
  \texttt{\{ytuan, william\}@cs.ucsb.edu, wewei@google.com} \\}

\date{}

\begin{document}
\maketitle
\begin{abstract}
Dialogue generation has been successfully learned from scratch by neural networks, but tends to produce the same general response, e.g., ``what are you talking about?’’, in many conversations. To reduce this homogeneity, external knowledge such as the speaker's profile and domain knowledge is applied as an additional condition to diversify a model’s output. The required knowledge to develop an effective conversation, however, is not always available, which is different from prior work’s assumption that a model always has acquired sufficient knowledge before chatting. This problem can be detrimental when applying a dialogue model like this chatting online with unconstrained people and topics, because the model does not have the needed knowledge. To address this problem, we propose InjK, which is a two-stage approach to \underline{Inj}ect \underline{K}nowledge into a dialogue generation model. First, we train a large-scale language model and query it as textual knowledge. Second, we frame a dialogue generation model to sequentially generate textual knowledge and a corresponding response. Empirically, when a dialogue generation model can only access limited knowledge, our method outperforms prior work by producing more coherent and informative responses.
\end{abstract}

\section{Introduction}

Research in dialogue generation aims to develop machines that can vividly converse with humans. One predominant method to solve this task is learning a neural network from large-scale real conversations~\cite{vinyals2015neural}. However, this approach creates the problem that a generated response tends to be general (e.g., ``I don’t know.’’ and ``What are you talking about?’’ are responses acceptable for many cases)~\cite{li2016diversity}. One reason is that there are diverse valid responses to the same dialogue, a method without proper design can only learn one vague response instead of memorizing multiple possibilities.

To generate more informative responses, recent works have collected data with external knowledge for reference, such as Persona-Chat~\cite{zhang2018personalizing, dinan2020second}. It provides textual profiles of the speakers involving in dialogues, e.g., ``Speaker1’s favorite sport is ultimate frisbee.’’ These profiles are often taken as the extra conditions for both training and testing the neural models. In this way, we can prevent vague responses but add specified information to them.

\begin{figure}[t!]
    \centering
    \includegraphics[width=.95\linewidth]{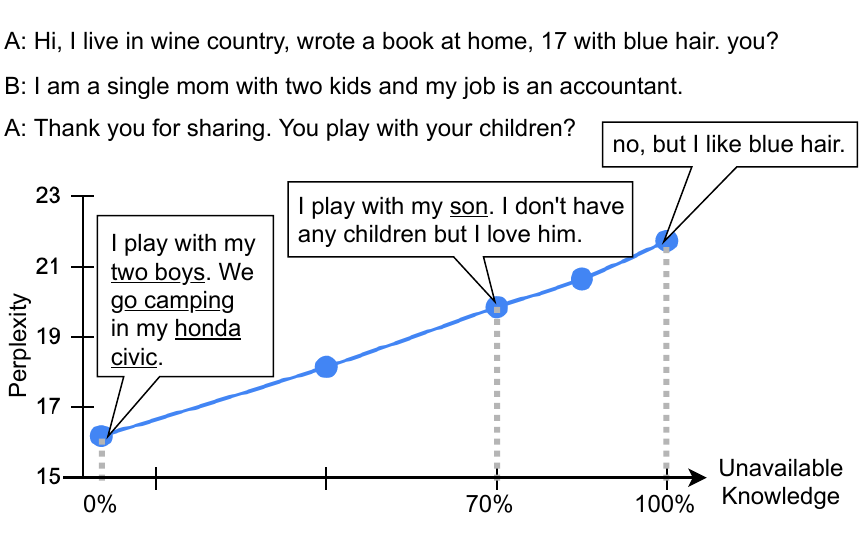}
    \caption{An example of a knowledge-grounded dialogue generation model test with insufficient knowledge, where underlines indicate knowledge related words. The model generates lower quality responses and rises perplexity when less knowledge is available.}
    \label{fig:intro}
\end{figure}

However, in a real conversation, we do not always have sufficient knowledge of a discussed topic as those artificially designed datasets. For instance, we do not know a stranger’s background in advance, nor do we know a song’s production process for it is often not publicly available.
Given that the accessible knowledge is actually limited, we conducted an experiment simulating this insufficiency in Figure~\ref{fig:intro}. The generated response is deviated from the conversation, resulting in much higher perplexity.
This inspires us to tackle the problem of insufficient knowledge in dialogue generation.

To solve knowledge insufficiency, we propose InjK, a unified method that does not require external knowledge in datasets. This approach formulates dialogue generation as a two-stage process. First, given an ongoing conversation, we conjecture possible knowledge in a learned domain that can be used for this dialogue. Second, the conjectured knowledge is used with the conversation to predict the response. More technically, the approach disentangles semantics information from the input message, maps it to a knowledge domain, then takes the mapped one as the knowledge for response generation. This process can model the stochastic causality from unobserved knowledge to response, thus improving the contained information in responses.

We analyze the performance of state-of-the-art and our proposed approach on this newly proposed task by how informative the generated responses are when limited knowledge for reference. We show that prior methods will collapse when not enough amounts of knowledge can be referred to, while InjK performs well less regarding to how much knowledge is given. The results demonstrate that considering the knowledge inefficiency in training methods for dialogue generation can improve user experience in real-scenarios, in which little knowledge can be obtained beforehand.

Our contributions are:
\begin{itemize}
    \item Discussion and simulation of knowledge insufficiency in dialogue generation.
    \item Proposal of a stochastic causal model to learn latent knowledge variable.
    \item Evaluation of knowledge insufficiency in inference time, showing that our proposed model achieves overall the best on two datasets.
\end{itemize}

\section{Related Work}
After the success of learning neural conversation models~\cite{vinyals2015neural,serban2016building}, people have explored to generate more informative responses.

Early approaches focused on learning more diverse responses by introducing latent variables.
\cite{li2016diversity, shao2017generating} re-ranked possible responses according to the mutual information between predictions and inputs.
\cite{bowman-etal-2016-generating, serban2017hierarchical, zhao2017learning, gao2019discrete} matched the posteriors and priors of dialogue generation as variational autoencoders~\cite{kingma2014auto} to reduce lost information in context.
\cite{li2017adversarial, xu2017neural, zhang2018generating, xu2018diversity, tuan2019improving} utilized adversarial learning~\cite{goodfellow2014generative} that was able to learn less vague distribution, thus providing more diverse results.
These approaches target to uncover underlying knowledge in conversations, but not injecting new information.

Beyond modifying models, recent predominant approaches were annotating conversations with speakers' profiles~\cite{li2016persona, zhang2018personalizing, shum2019sketch, zheng2019pre, dinan2020second} and topic-related knowledge~\cite{ghazvininejad2018knowledge, dinan2018wizard, zhou2018commonsense, moon2019opendialkg, raghu2019disentangling, tuan2019dykgchat}. By increasing information in context, the responses were more informative than only revising models. Since this advantage needed intensive context-knowledge-response pairs in data, \cite{Zhao2020Low-Resource, li2020zero} devised methods to reduce these required annotations in training set.
However, the lack of knowledge in inference time is still an issue.
This neglect can cause severe distribution shift problem when releasing those dialogue models for applications.


\begin{figure}[t]
    \centering
    \includegraphics[width=.65\linewidth]{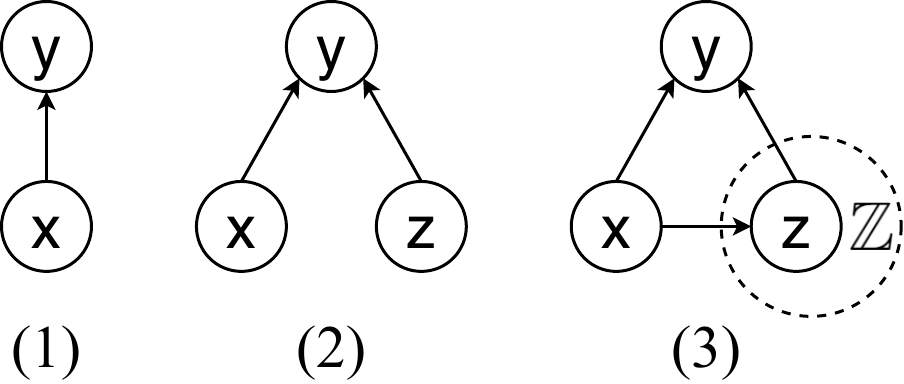}
    \caption{The graphical models of (1) standard dialogue generation (2) knowledge-grounded dialogue generation (3) our proposed approach, where $x$ represents the dialogue history, $y$ represents the response, $z$ represents the paired knowledge in (2), which is not always accessible, or the inferred knowledge in (3), and $\mathbb{Z}$ represents the knowledge domain.}
    \label{fig:graphical}
\end{figure}

\section{Preliminaries}

\subsection{Dialogue Generation}
Given a dialogue history $x$, which consists of multiple previous turns in a conversation, the task \emph{dialogue generation} aims to predict the next turn $y$, which is a sequence of tokens. When modeling dialogue generation by neural networks, a standard method maximizes the log-likelihood of generating $y$ by the loss function $L=\log P(y|x)=\Sigma_t \log P(y_t|x,y_{<t})$~\cite{vinyals2015neural}.

Some dialogue generation models assume that an external knowledge base is available. For each conversation pair $(x,y)$, a paired knowledge sample is either retrieved from the knowledge base or annotated. Here we focus on examples of textual knowledge and denote a sample knowledge sentence as $z$.\footnote{Other forms of knowledge, e.g., knowledge graph, is left for future work and can also be framed as sentences using templates.}
Given this paired $z$, knowledge-grounded dialogue models are generally optimized by the loss function $L=\log P(y|x,z)$~\cite{ghazvininejad2018knowledge} that gains benefits from $z$ to restrict the variance of $y$.

Their graphical models are plotted in Figure~\ref{fig:graphical}(1)(2), where the standard model only learns the causal relation from $x$ to $y$ but the knowledge-grounded model attributes $y$ to $x$ and $z$.

\subsection{Knowledge Insufficiency}
In a real conversation, we may not have the required knowledge to develop an engaging response. For example, if there is a conversation about ``Game of Thrones'', but we have not seen ``Game of Thrones'' and do not know what it is, we are less likely to engage in this conversation.
We call this the \emph{insufficient knowledge} problem. 
With this lack, a person might just copy a familiar term in the conversation and said ``What is the board game?'' (however, ``Game of Thrones'' is a TV series), thus deviating from the conversation.

Conversations with insufficient knowledge are common in natural discussion, where an interlocutor might not have adequate knowledge or experience regarding the topic at hand and need to produce exchangeable information relying on our cognitive system~\cite{pask1976conversation, hammersley2003conversation, turnbull2003language}.
We describe this scenario in Figure~\ref{fig:graphical}(3), where $\mathbb{Z}$ imitates the cognitive system to produce information $z$ with the $x$'s assistance, while the model attributes $y$ to $x$ and the guessed $z$ without access to external annotations.

\begin{figure}[t]
    \centering
    \includegraphics[width=.9\linewidth]{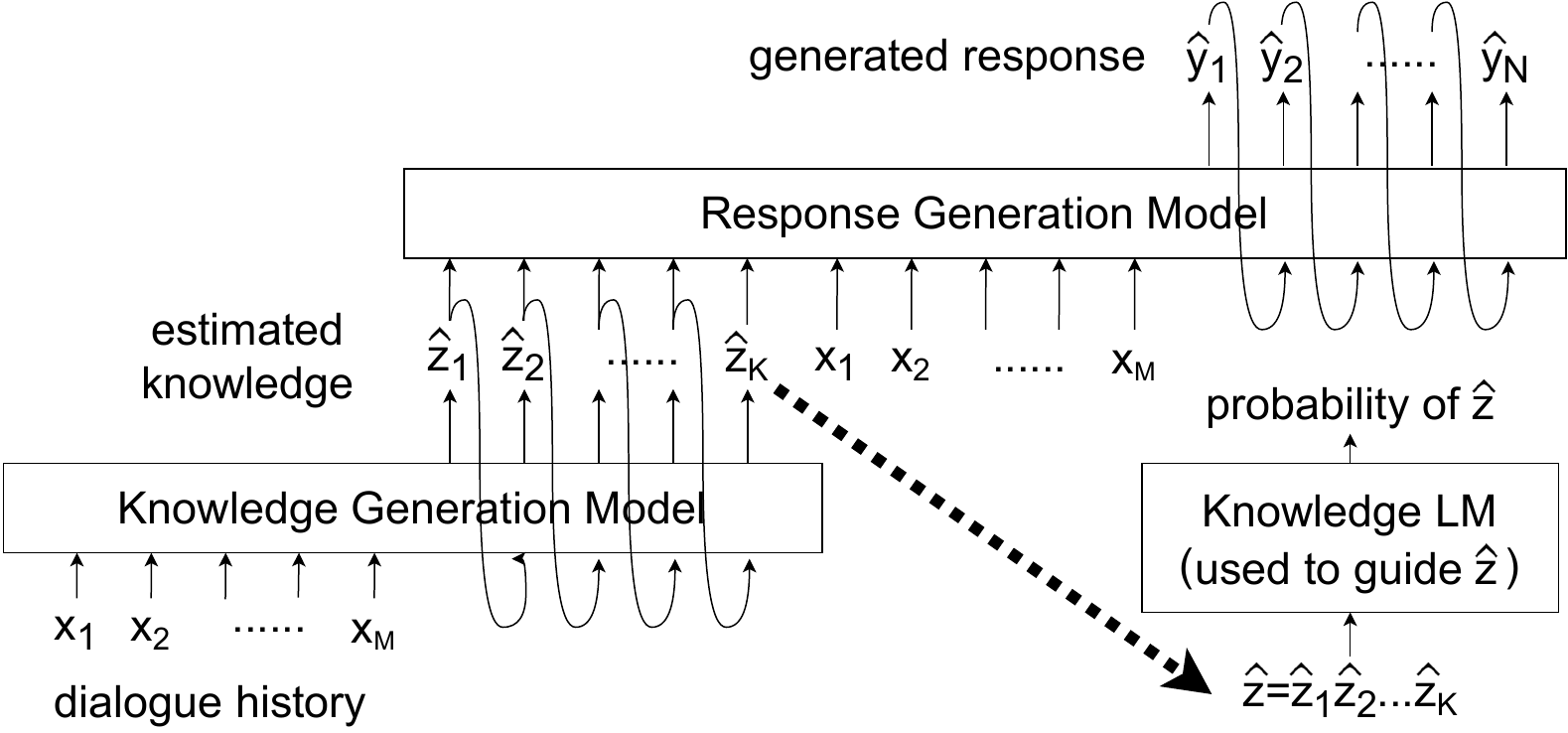}
    \caption{The proposed method is composed of three modules. First the dialogue history $x=x_1x_2...x_M$ is fed into knowledge generate model. The model then infers an estimated knowledge $\hat{z}=\hat{z}_1\hat{z}_2...\hat{z}_K$. Finally, $\hat{z}$ and $x$ are concatenated and fed into the response generation model to produce the response $\hat{y}=\hat{y}_1\hat{y}_2...\hat{y}_N$. During training, the estimated knowledge $\hat{z}$ is fed into the knowledge LM to maximize its probability.}
    \label{fig:modules}
\end{figure}

\section{Knowledge-injected Dialogue Model}
\label{sec:method}

We propose a knowledge-injected dialogue model (InjK) that instead of ``grounding on'' external knowledge~\cite{ghazvininejad2018knowledge, zhou2018commonsense, Zhao2020Low-Resource, li2020zero} ``injects'' knowledge into a model, so that this model can be test with limited accessible knowledge.

\subsection{Motivation}
We observe that when a talker is having a conversation, she carries assumptions about the world that is somewhat independent from what she heard in this conversation. The talker then develops a response by combining her received messages in this dialogue and her realization of the world. For example in a conversation, someone said ``I once road on The Royal Blue train from New York to D.C'', the talker may have her ``assumed knowledge'' (may not be the truth) that ``People can see beautiful views from the windows of a train'', therefore she said ``Oh that sounds really nice. I bet there was a lot of scenery and blue skies.''

In this illustrative example, we can summarize the relationships among ``assumed knowledge'' $z$, dialogue history $x$, and response $y$ as the following three points that were said to build causality~\cite{hlavavckova2007causality} from $x$ to $y$ and $z$ to $y$.
\begin{itemize}
    \item $y$ follows after $x$ and $z$ (time-ordering)
    \item $y$ relates to $x$ and $z$ (covariation)
    \item after removing $x$, $y$ still relates to $z$ ($x$ and $z$ are confounding variables to $y$)
\end{itemize}

Motivated by this observation, we cast dialogue generation as a stochastic causal model~\cite{pearl1987evidential} and take a talker's assumed knowledge as an unobserved variable $z$ that meets the above three conditions. We define the causality among $x$,$y$,$z$ in equation~\ref{eq:causal}, where $\mathbb{Z}$ is an abstract set of knowledge.
\begin{equation}\label{eq:causal}
    \begin{split}
        P(y|x) & = \mathbb{E}_{z\in\mathbb{Z}}P(y|x,z)P(z|x)\\
    \end{split}
\end{equation}

\subsection{Model Details}
The framework of our proposed InjK is depicted in Figure~\ref{fig:modules} and has three modules: a knowledge language model (LM) $\theta$, a knowledge generation model $\sigma$, and a response generation model $\phi$, where $\theta$, $\sigma$, $\phi$ are their parameters.
We denote a textual knowledge as $z=\{z_t\}_1^{|z|}\in \mathbb{Z}$, where each $z_t$ is a token and $|z|$ is the length of this knowledge.

\paragraph{Test Phase.}
We first feed a sequence $x$ into the knowledge generation model $\sigma$ and sample an estimated knowledge sequence $\hat{z}$ from a probability distribution $P_\sigma(\hat{z}|x)$ that is parameterized by $\sigma$. Next, we take $x$ and $\hat{z}$ as the inputs to predict a response $\hat{y}$ from distribution $P_\phi(\hat{y}|x,\hat{z})$ using the response generation model.
Note that we predict $\hat{z}$ and $\hat{y}$ token-by-token until the eos-of-sentence symbol is generated.
We formulate this generation process in test phase as equation~\ref{eq:test_process}, where the symbol $\sim$ means {\it sampled from} a distribution.
\begin{equation}\label{eq:test_process}
    \begin{split}
        \hat{z} & \sim \Pi_t P_\sigma(\hat{z}_t|x, \hat{z}_{<t})\\
        \hat{y} & \sim \Pi_t P_\phi(\hat{y}_t|x,\hat{z}, \hat{y}_{<t})\\
    \end{split}
\end{equation}

\paragraph{Training Phase.}
Similar to Test Phase, we sample $\hat{z}$ from $P_\sigma(\hat{z}|x)$, but use $x$, $\hat{z}$, and the previous tokens in ground-truth response $y_{<t}$ to predict token $y_t$ on $t$-th step. We then refine equation~\ref{eq:causal} and define the general objective as
\begin{equation}\label{eq:general-obj}
    \begin{split}
        L(\phi,\sigma) & = \mathbb{E}_{z\in\mathbb{Z}}P_\sigma(z|x) \Pi_t P_\phi(y_t|x,z,y_{<t})\\
    \end{split}
\end{equation}
Next, we relax the constraint $z\in\mathbb{Z}$ by regularizing a deduced $z$ lying in $P_\theta(\mathbb{Z})$, which is paramerized by a knowledge LM $\theta$ that simulates the distribution of possible knowledge.
Particularly, we train a knowledge LM by maximizing log-likelihood on a collection of textual knowledge that depends on used dataset.
We then adopt the mode-seeking direction of Kullback-Leibler divergence(KLD)~\cite{kullback1951information,agarwal2019learning} and formulate the loss functions as

\paragraph{Generation Models:}
\begin{equation}\label{eq:loss}
    \begin{split}
        L(\phi,\sigma) = & -\log P_\phi(y|x,z) P_\sigma(z|x) \\
        & + D_{KL}(P_\sigma(z|x)||P_\theta(z))\\
    \end{split}
\end{equation}
{\bf Knowledge LM:}
\begin{equation}\label{eq:klm-loss}
    \begin{split}
        L(\theta) & = -\Sigma_{t=1}^{|z|} \log P_\theta(z_t|z_{<t})
    \end{split}
\end{equation}

Detailed derivations are presented in Supplementary Material.

\begin{algorithm}[t]
  \caption{\textsc{InjK Optimization}}
  \label{alg}
  \KwData{dialogues $D$, textual knowledge $\mathbb{Z}$}
  \KwIn{initialize knowledge LM ($\theta$), knowledge generation model $\sigma$, response generation model $\phi$}
  \KwIn{hyperparamters $\alpha$, $\beta$, $\gamma$}
  \For{each pretraining iteration}
  {
     sample a mini-batch $z$ from $\mathbb{Z}$\\
     compute $\nabla_\theta$ as Eq.~\ref{eq:grads}\\
     update $\theta$ with learning rate $r_\theta$ by\\
     \Indp{
       $\theta \leftarrow \phi + r \nabla_\theta$\\
     }
  }
  \For{each training iteration}
  {
     sample a mini-batch $(x,y)$ from $D$\\
     sample $\hat{z}$ from $P_\sigma(z|x)$\\
     compute $\nabla_\phi$ as Eq.~\ref{eq:grads}\\
     compute $\nabla_\sigma$ as Eq.~\ref{eq:sigma-grad}\\
     update $\phi$ and $\sigma$ with learning rate $r$ by\\
     \Indp{
       $\phi \leftarrow \phi + r \nabla_\phi$\\
       $\sigma \leftarrow \sigma + r \nabla_\sigma$\\
     }
  }
\end{algorithm}

\subsection{Optimization Challenges}
We propose two further steps to avoid the intractability of (a) $D_{KL}$ and (b) $P_\sigma(z|x)$ in equation~\ref{eq:loss}.

\paragraph{$D_{KL}$}
When computing the KLD term in equation~\ref{eq:loss}, we have to multiply the probabilities of each token in $z$, which will grow exponentially w.r.t $|z|$ and make this KLD term intractable.
In practice, we substitute the KLD with an upper bound in equation~\ref{eq:kld-bound} and prove this bound in Supplementary Material.
\begin{equation}\label{eq:kld-bound}
    \begin{split}
        D_{KL}(P_\sigma(z|x)& ||P_\theta(z)) \\ \leq \Sigma_{t} & D_{KL}(P_\sigma(z_t|x,z_{<t})||P_\theta(z_t|z_{<t}))
    \end{split}
\end{equation}

\paragraph{$P_\sigma(z|x)$}
We apply policy gradient~\cite{sutton2018reinforcement, ranzato2015sequence} to prevent from the non-differentiable $z$ in $P_\phi(y|x,z)P_\sigma(z|x)$, since $z$ is a sequence of discrete variables sampled from distribution $P_\sigma(z|x)$.
We cast this problem as reinforcement learning by framing the $P_\phi(y|x,z)$ as the reward after taking a sequence of actions $\{z_t\}_1^{|z|}$ that are selected from the policy $P_\sigma(z|x)$. To reduce the high variance of policy gradient update, we subtract the expected reward of taking action $z_t$ by a baseline $b$. The gradient of $\sigma$ from $P_\phi(y|x,z)$ can thus be formulated as
\begin{equation}\label{eq:pg-loss}
    \begin{split}
        \nabla_\sigma 
            & = \Sigma_{t=1}^{|z|} (P_\phi(y|x,z)-b) \nabla_\sigma \log P_\sigma(z_t|x,z_{<t})\\
    \end{split}
\end{equation}

Overall, we approximate the gradient of $\sigma$ by combining equation~\ref{eq:kld-bound} and equation~\ref{eq:pg-loss} as
\begin{equation}\label{eq:sigma-grad}
   \begin{split}
        \nabla_\sigma & = \\
        & \Sigma_{t=1}^{|z|} [ \beta (P_\phi(y|x,z)-b) \nabla_\sigma \log P_\sigma(z_t|x,z_{<t})\\
        & - \gamma \nabla_\sigma D_{KL}(P_\sigma(z_t|x,z_{<t})||P_\theta(z_t|z_{<t}))]\\
    \end{split}
\end{equation}
while we compute the gradients of $\theta$ and $\phi$ by maximizing log-likelihood as
\begin{equation}\label{eq:grads}
    \begin{split}
        \nabla_\theta & = \Sigma_{t=1}^{|z|} \nabla_\theta \log P_\theta(z_t|z_{<t})\\
        \nabla_\phi & =  \alpha \Sigma_{t=1}^{|y|} \nabla_\phi \log P_\phi(y_t | x, z, y_{<t})\\
    \end{split}
\end{equation}
where $\alpha$, $\beta$, $\gamma$ are coefficients to tune the importance of each term. Empirically, the gradients work well when $\alpha$, $\beta$ are set to 1, and $\gamma\in [0.1,1]$. The parameters $\theta$, $\phi$, $\sigma$ are then updated by gradient ascent.

The learning algorithm is summarized in Algorithm~\ref{alg}.

\begin{figure*}[t]
    \centering
     \begin{subfigure}[b]{.95\linewidth}
         \centering
         \includegraphics[width=\linewidth]{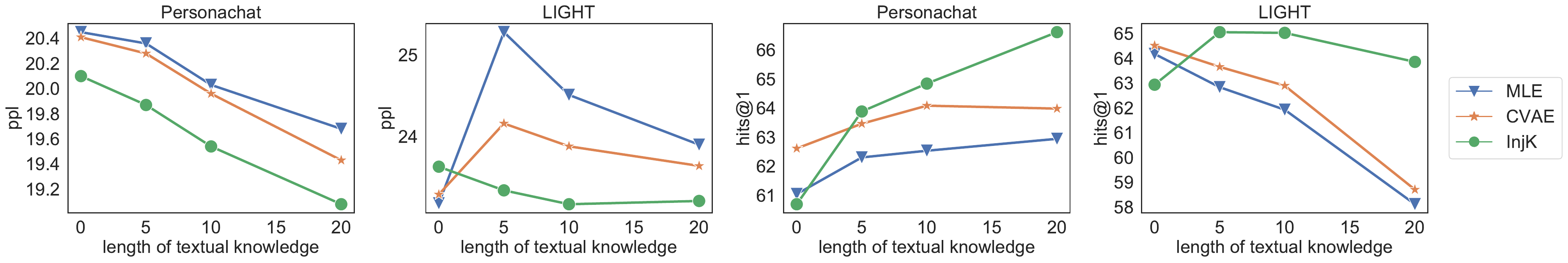}
         \caption{Comparison to models trained without paired knowledge.}
         \label{fig:y equals x}
     \end{subfigure}
     \hfill
     \begin{subfigure}[b]{.95\linewidth}
         \centering
         \includegraphics[width=\linewidth]{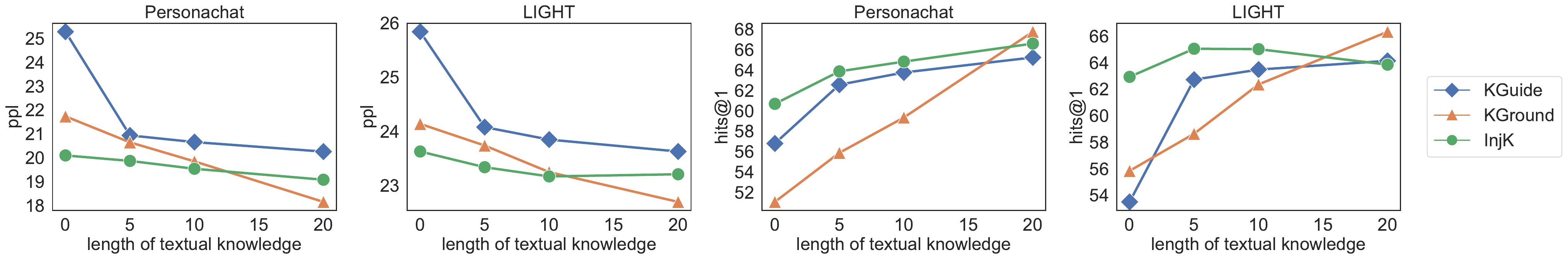}
         \caption{Comparison to models trained with paired knowledge.}
         \label{fig:y equals x}
     \end{subfigure}
\caption{Evaluation of insufficient knowledge (length of textual knowledge) by perplexity (ppl) and hits@1 on Personachat and LIGHT datasets.}
\label{fig:curves}
\end{figure*}

\begin{table*}[t]\small
    \centering
    \begin{tabular}{l|ccccc|ccccc}\toprule[1pt]
        & \multicolumn{5}{c}{Personachat} & \multicolumn{5}{|c}{LIGHT}\\\midrule[0.5pt]
        Model & PPL & Hits@1 & BLEU-1 & BLEU-2 & Average
         & PPL & Hits@1 & BLEU-1 & BLEU-2 & Average\\\midrule[0.5pt]
        MLE & 19.9 & 62.5 & 11.13 & 4.75 & 0.8657 
            & 24.5 & 61.9 & 8.84 & 2.90 & 0.8566\\
        CVAE & 20.0 & 64.1 & 11.19 & 4.87 & 0.8644 
            & 23.9 & 62.9 & 8.97 & 2.89 & 0.8622\\
        KGuide & 20.7 & 63.8 & 11.02 & 4.87 & 0.8646 
            & 23.9 & 63.5 & 9.06 & 2.98 & 0.8639\\
        KGround & 20.0 & 59.4 & 10.61 & 4.89 & 0.8608
            & 23.3 & 62.4 & 8.72 & 2.88 & 0.8641\\
        InjK & \bf 19.5 & \bf 64.9 & \bf 11.38 & \bf 4.92 & \bf 0.8658 
            & \bf 23.2 & \bf 65.0 & \bf 9.42 & \bf 3.12 & \bf 0.8657\\\bottomrule[1pt]
    \end{tabular}
    \caption{Automatic evaluation on Personachat and LIGHT with limited textual knowledge (length=10).} 
    \label{tab:results}
\end{table*}

\section{Experiments}

We evaluate the knowledge insufficiency problem in test phase on dialogues that should be informative but lack knowledge for reference.

\subsection{Datasets}

To simulate usual scenarios that people are unfamiliar with a discussed topic, we design two datasets that respectively does introduction and role-playing game tasks, but take speakers' profiles as the knowledge.

\paragraph{Conversations talking about profiles.}
We formulate Personachat~\cite{zhang2018personalizing, dinan2020second} as two-speaker dialogues with a textual description of whom is to respond. The description is composed of multiple sentences about personalities. When collecting the dialogues, the speakers were asked to \emph{get to know each other} following their assigned profiles. The resulting amounts are 131438 / 7801 / 7504 for train/valid/test sets.

\paragraph{Conversations acting as profiles.}
We frame LIGHT~\cite{urbanek2019learning} as two-speaker dialogues with textual descriptions of characters in an adventure game. When collecting the dialogues, the speaker was asked to act as the assigned profiles to converse with other characters. The resulting amounts are 102109 / 6123 / 12272 for train/valid/test sets.

We simulate the cases that the required knowledge is not accessible by masking the profiles. The masked datasets are then challenging for that we need to learn the underlying knowledge of a response without annotations but infer an engaging response that (a) directly talks about the knowledge  and (b) implicitly acts following the knowledge.

\subsection{Baselines}

We modified state-of-the-art dialogue models for this task as baselines. For fairness, we implemented them based on the same transformer architecture.
{\bf MLE}~\cite{vinyals2015neural}, stands for maximum likelihood estimation, that trains a neural conversation model with only dialogue history as inputs, and then predicts the response. The loss function is $L=-\log P(y|x)$.
{\bf CVAE}~\cite{zhao2017learning, serban2017hierarchical, gao2019discrete}, stands for conditional variational autoencoder, that first predicts a discrete latent variable $z'$, which is trained as the posterior of the dialogue history and the response. Then the method uses $z'$ and dialogue history to predict the response. The loss function is $L=-\log P(y|x,z') + D_{KL} (P(z'|x)||P(z'|x,y))$.
{\bf KGuide}~\cite{zhao2017learning} that trains a model to predict the paired knowledge $\tilde{z}$ and then takes the prediction and dialogue history as inputs to generate the response. The loss function is $L=-\log P(y|{z,x}) - \log P(\tilde{z}|x)$.
{\bf KGround}~\cite{wolf2019transfertransfo} that trains a model by concatenating the paired knowledge $\tilde{z}$ and the dialogue history, and takes them as inputs to generate the response. The loss function is $L=-\log P(y|{\tilde{z},x})$.
We prevent some recent work such as \cite{Zhao2020Low-Resource, li2020zero} from comparison since they require the presence of a large knowledge base in test time, which is not fair to compete in our setup.

\begin{table*}[t]\small
    \centering
    \begin{tabular}{p{0.5\linewidth}|p{0.43\linewidth}}\toprule[1pt]
        \bf Input & \bf Response \\\midrule[0.5pt]
        \makecell[l]{
        A: wow, you need a hobby to get away, like jujitsu or running.\\
        B: i like running, but i have no time.\\
        A: well i sure can understand that. running does relieve stress\\ though\\
        B: it does, i like to wear my favorite color blue when i do run.\\
        A: when you {\bf have the time where do you run}?}
        & \makecell[l]{
        MLE: i {\bf run to the gym} in my spare time.\\
        CVAE: i like to {\bf run in the woods}.\\
        KGuide: {\bf anywhere but the park}.\\
        KGround: the farm i grew up on is a farm, but i like to\\ {\bf run in the fields}.\\
        InjK: i {\bf run to the park near my house}.}\\\bottomrule[1pt]
        
        
        \bf Input & \bf Response\\\midrule[0.5pt]
        \makecell[l]{
        A: i mean that's where i live now how about you\\
        B: big city now. helps with {\bf my career in singing}\\
        A: so {\bf what do you see any music you sing}\\\\\\\\
        }
        & \makecell[l]{
        MLE: i like {\bf rock} music and i sing in front of a bunch\\ of people.\\
        CVAE: {\bf rock, pop, rap, alternative}. i also like to hike.\\
        KGuide: i like to sing and sing on the weekends.\\
        KGround: mostly horror movies. i watch a lot of them.\\
        InjK: i {\bf sing in a church choir}.}\\\bottomrule[1pt]
        
        

    \end{tabular}
    \caption{Examples of generated responses with limited textual knowledge (length=10) of Personchat. Bold texts indicate the related knowledge parts in the inputs and responses.}
    \label{tab:case-study}
\end{table*}

\subsection{Implementation Details}

We implemented the models based on transformer~\cite{vaswani2017attention,liu2018generating} and started the training from the GPT model (110M)~\cite{radford2018improving} pretrained on BooksCorpus~\cite{zhu2015aligning}.\footnote{We implement the code with~\cite{wolf-etal-2020-transformers} and refer to \url{https://github.com/huggingface/transfer-learning-conv-ai}.}
We append to the inputs with \emph{speaker embeddings}, which indicate a token is of whom.
Besides log-likelihood loss, the last hidden state is passed to a fully-connected layer to classify if a response is the ground-truth compared with randomly sampled candidates.
The model is trained on one Titan RTX with a batch size 4, a maximum input length 512, and a maximum output length 20.
The $\alpha$,$\beta$,$\gamma$ in Algorithm~\ref{alg} are set as 1 but $\gamma$ for LIGHT is set as 0.1.
The learning rate is 0.0000625 with linearly decay to 0 in 3 epochs and the optimizer is AdamW~\cite{loshchilov2018decoupled}.

\subsection{Results}
We automatically evaluate the models by perplexity (PPL) of the ground-truth responses and Hits@1 among 20 candidate responses following~\cite{zhang2018personalizing,urbanek2019learning}. PPL is the lower the better while Hits@1 is the higher the better. We also provide BLEU and average embedding cosine similarity~\footnote{\url{https://github.com/Maluuba/nlg-eval}} in Table~\ref{tab:results}, where InjK performs better or comparable to the baselines with limited knowledge available.

In Figure~\ref{fig:curves}, we simulate the knowledge insufficiency problem by testing a dialogue generation model with limited textual knowledge. The x-axis in the plots is specifically the length of a randomly selected sub-sequence in the paired knowledge.

As shown in Figure~\ref{fig:curves}, most PPL curves decline and most Hit@1 curves rise along with the amount of accessible knowledge. This validate our assumption that $z$ is a confounding variable to $y$ in Section~\ref{sec:method}, thus enabling our further experiments.
When comparing InjK with MLE and CVAE that also do not trained with paired knowledge, as presented in sub-figures~\ref{fig:curves}(1)(2), we observe that InjK outperforms MLE and CVAE along different extents of knowledge insufficiency. This demonstrates that InjK can learn how to more efficiently utilize knowledge even when no annotations are available for training.
When comparing InjK with KGuide and KGround that are trained with annotations, as seen in sub-figures~\ref{fig:curves}(3)(4), we observe that InjK outperforms KGuide and KGround when test with limited obtainable knowledge. This shows that KGuide and KGround can suffer from severe dataset shift when they are trained with annotations but test with limited amounts (we will further analyze this shift in Section~\ref{sec:gap-analysis}). This comparison also indicates that InjK can perform well regardless of knowledge insufficiency.

\begin{table}[t]\small
    \centering
    \begin{tabular}{lcccccc}\toprule[1pt]
        & \multicolumn{3}{c}{Coh} & \multicolumn{3}{c}{Info}\\
        & win & tie & $\kappa$ & win & tie & $\kappa$\\\midrule[0.5pt]
        InjK vs MLE & 46.0 & 17.1 & .59 & 38.6 & 23.6 & .52\\
        InjK vs CVAE & 45.4 & 16.1 & .49 & 34.0 & 27.3 & .53\\\bottomrule[1pt]
    \end{tabular}
    \caption{Human evaluation of coherence (Coh) and informativeness (Info) on Personachat without parallel profiles. Note that Info is designed to be independent from Coh by only considering responses but no inputs, so Coh and Info should be used together to access if a response gives needed information.}
    \label{tab:human-eval}
\end{table}

\subsection{Human Evaluation}
We conducted human evaluation mainly on Personachat and focused on models trained without knowledge.
We adopted softmax sampling with temperature set to $0.7$ to generate the responses~\cite{wolf2019transfertransfo} {\it with no} knowledge by each model and randomly selected $200$ history-response pairs for evaluation.
Next, we invited 3 native speakers with over $95\%$ approval rate on Amazon Mechanical Turk (MTurk)~\cite{buhrmester2016amazon} to judge the models' performance.
They were asked to compare two generated responses according to coherence (Coh) and informativeness (Info).
The instruction is given as, (a) {\bf Coh:} the sentence is coherent to the semantics of the conversation history, and is a proper response to the conversation history; (b) {\bf Info:} the sentence provides some information about the speakers, scenes, or background. The sentence should not be a general question nor be meaningless, such as ``I don't know''.
They were asked to choose which response is better on each criterion or the responses are tied.

The human evaluation results are listed in Table~\ref{tab:human-eval}, where the annotators' agreements $\kappa$ are measured by Fleiss' kappa~\cite{fleiss1971measuring}. 
According to an interpretation of Fleiss' kappa~\cite{landis1977measurement}, the annotators' agreements are moderate.
The results show that InjK has better coherence than both MLE and CVAE, but their contained information is about the same level when testing without paired knowledge. This may because that without the guidance of knowledge, MLE and CVAE generate some randomized semantics that annotators may take the semantics to be informative, but the responses can actually distract users from the conversation.

\begin{figure}[t]
    \centering
    \includegraphics[width=.98\linewidth]{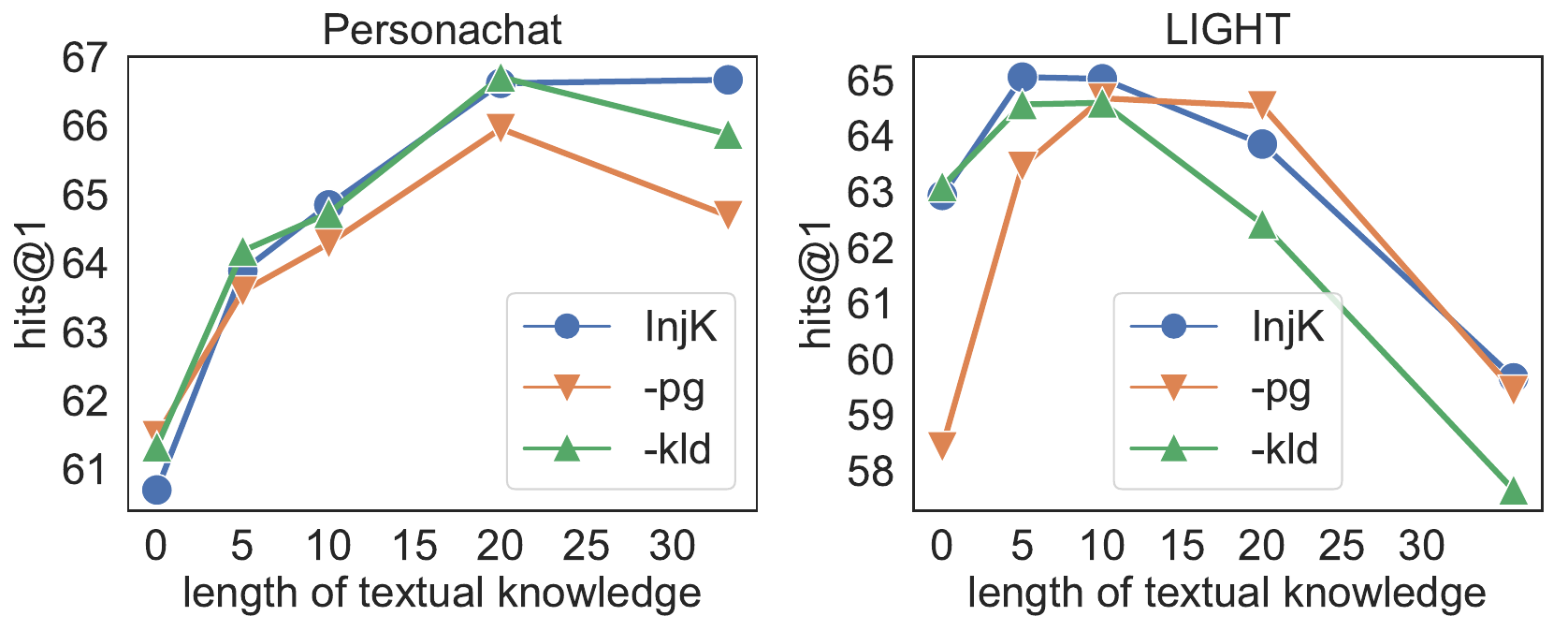}
    \caption{The ablation study of InjK. {\it -pg} and {\it -kld} are the results when the policy gradient and KLD terms drop off from equation~\ref{eq:pg-loss}.}
    \label{fig:ablation}
\end{figure}

\section{Analysis}
We study generated examples in subsection~\ref{sec:case-study}, the necessity of policy gradient and KLD terms in equation~\ref{eq:pg-loss} in~\ref{sec:ablation}, and our proposed knowledge gap in subsection~\ref{sec:gap-analysis}.

\subsection{Qualitative Analysis}\label{sec:case-study}
Some examples are presented in Table~\ref{tab:case-study}, where MLE and CVAE generate acceptable response but slightly deviate from the focus, KGuide and KGround generate diverse ngrams but less likely to maintain the flow, InjK shows its possibility to exchange precise information in conversations.

\subsection{Ablation Study}\label{sec:ablation}
We respectively drop the policy gradient (pg) and KLD terms off from equation~\ref{eq:pg-loss} in InjK and present the results in Figure~\ref{fig:ablation}.
As can be seen, the impacts of deleting pg and KLD terms vary on Personachat and LIGHT, but including both terms generally achieves the best performance.

\subsection{Knowledge Gaps Analysis}
\label{sec:gap-analysis}

To analyze the slope of curves in Figure~\ref{fig:curves}, we define {\bf knowledge gap} as the variance of PPL w.r.t. knowledge amounts to \emph{measure the importance of the presence of sufficient external knowledge to a dialogue model}.
Formally, we formulate the knowledge gap of a dialogue model as the standard deviation of PPL over tests with different lengths of textual knowledge.
Moreover, we define the average knowledge gap in training phase of a dataset as
\begin{equation}\label{eq:kg-gaps-in-train}
    \frac{1}{num(k)}\Sigma_k stdev_i(\{PPL(P_{\phi^{(i)}} (y|x,z^k))\}_i)
\end{equation}
where $\phi^{(i)}$ is the parameters of $i$-th model, $stdev_i$ means standard deviation over $i$, and $k$ denotes the length of textual knowledge. This training phase knowledge gap is the average value over $k$ of the PPL standard deviation across models and aims to quantify the importance of using knowledge for training.

We define the knowledge gap in testing phase of a dataset as
\begin{equation}\label{eq:kg-gaps-in-test}
    \frac{1}{num(i)}\Sigma_i stdev_k(\{PPL(P_{\phi^{(i)}} (y|x,z^k))\}_k)
\end{equation}
which means the average value over models of the PPL standard deviation across length of textual knowledge.
This value targets to quantify the importance of using knowledge in test time.


\begin{figure}[t!]
\centering
\includegraphics[width=.98\linewidth]{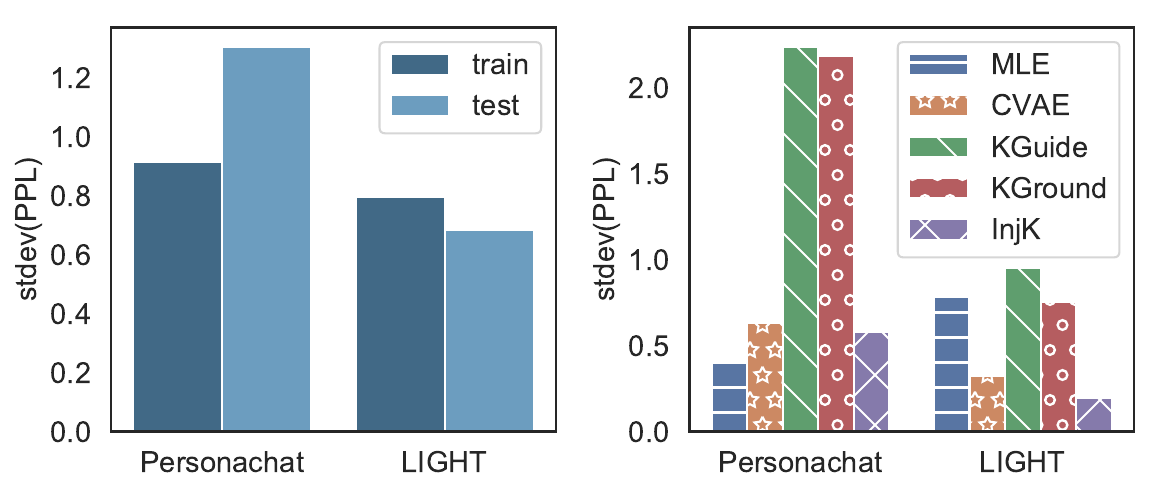}
\caption{Comparison of (1) the knowledge gaps in training and test phase, as well as (2) the knowledge gaps of different models, which are defined as equations~\ref{eq:kg-gaps-in-train} and \ref{eq:kg-gaps-in-test}, for each dataset.}
\label{fig:knowlegde_gaps}
\end{figure}

\paragraph{Knowledge gaps in training v.s test phases.}

As shown in Figure~\ref{fig:knowlegde_gaps}(1), the average knowledge gaps in test phase are comparable to the ones in training phase.
This indicates that test with sufficient knowledge is as influential as training, whereas prior work only focus on the problem in training time.

\paragraph{Knowledge gaps of different models.}
We draw the knowledge gaps of different models in Figure~\ref{fig:knowlegde_gaps}(2).
As can be seen, KGuide and KGround highly depend on the amounts of knowledge. This shows they fail to effectively utilize the dialogue history and unlikely generate knowledge as a cognitive process.
MLE, CVAE, and InjK, with narrower knowledge gaps, can optimize the usage of dialogue history.
Among them, InjK generally has the smallest knowledge gap and performs better in Figure~\ref{fig:curves}, showing that considering knowledge insufficiency in both training and test is more likely to embed cognition in a dialogue model.

\section{Conclusion}
We formulate the insufficient knowledge problem in a real conversation where speakers cannot access needed knowledge to respond. To deal with this problem, we propose InjK that injects knowledge into a dialogue model by regularizing a latent variable being knowledgeable.
Empirically, we analyze the impact of knowledge insufficiency in both training and testing phases.
The results show that InjK outperforms state-of-the-art baselines when limited knowledge is obtainable in applications.

\section{Ethical Considerations}
This work trains the models on crowd-source datasets from prior work, thus may resulting in some toxic knowledge in generated responses, such as bias and misinformation.


\bibliographystyle{acl_natbib}
\bibliography{main}

\clearpage

\appendix

\section{Derivation of KLD upper bound}

\begin{proof}\small
    \begin{equation}
        \begin{split}
            &   D_{KL}(P_\sigma(z|x)||P_\theta(z))\\
            &   = \Sigma_{z\in \mathbb{Z}}P_\sigma(z|x)\log\frac{P_\sigma(z|x)}{P_\theta(z)}\\
            &   = \Sigma_{z\in \mathbb{Z}}P_\sigma(z|x) [\log P_\sigma(z|x) - \log P_\theta(z)]\\
            &   = \underset{z\in \mathbb{Z}}{\Sigma}P_\sigma(z|x) [\underset{k}{\Sigma} \log P_\sigma(z_k|x,z_{1:k-1}) - \underset{k}{\Sigma} \log P_\theta(z_k|z_{1:k-1})]\\
            &   = \Sigma_{z\in \mathbb{Z}}\Sigma_k P_\sigma(z|x) \log \frac{P_\sigma(z_k|x,z_{1:k-1})}{P_\theta(z_k|z_{1:k-1})}\\
            &   \leq \Sigma_{z}\Sigma_{k} P_\sigma(z_k|x,z_{1:k-1})\log \frac{P_\sigma(z_k|x,z_{1:k-1})}{P_\theta(z_k|z_{1:k-1})}\\
            &   = \Sigma_{k} D_{KL}(P_\sigma(z_k|x,z_{1:k-1})||P_\theta(z_k|z_{1:k-1}))\\
        \end{split}
    \end{equation}
\end{proof}
The last inequality holds since $P_\sigma(z|x)=\prod_k P_\sigma(z_k|x,z_{1:k-1})$,
\begin{equation}
    P_\sigma(z|x) \leq P_\sigma(z_k|x,z_{1:k-1})\text{, for any $k$}
\end{equation}

\section{Derivation of InjK}
\label{sec:derivation}
Recall that in Method Section in the main content, we define notations $x$ for dialogue history, $y$ for response, and $z$ for the knowledge.
To clarify more, we define $\tilde{y}$ as the ground-truth response, $\hat{y}$ as the generated response, $\tilde{z}$ as the labeled knowledge, and $\hat{z}$ as the estimated knowledge. When using $y$ or $z$, we do not specify its status but represent it as the variable for description. Moreover, we denote the knowledge domain as $\mathbb{Z}$ and assume all $z\in\mathbb{Z}$.

Borrowing the concept of Causality~\cite{hlavavckova2007causality}, we can suppose $z$ is a confounding variable besides $x$ of $y$ since it contains extra information for $y$ that $x$ does not have.
However, when having no knowledge during training, $z$ is an unobserved variable that needs us to discover.
As the concept of Causality, variable $z$ is a cause to $y$ other than $x$ if and only if $P(y|x,z) > P(y|x)$. This means $z$ gives {\it extra, useful} information to infer $y$.
We can rewrite the inequality in a more realistic scenario as the following.
\begin{lemma}
    given $(x,\tilde{y},\tilde{z})\sim P_D$, where $P_D$ is the real data distribution, and $\tilde{z}$ contains information useful to $\tilde{y}$ and different from $x$.
    \begin{equation}
        \max_\phi P_\phi(\tilde{y}|x) < \max_\theta P_\theta(\tilde{y}|x,\tilde{z})
    \end{equation}
\end{lemma}
where $\phi$ is the parameters of the generative model that only takes $x$ as input, and $\theta$ is the ones that also takes $\tilde{z}$ as input.

We further define a residual function $\epsilon$ as follows to quantify the difference between $\max_\phi P_\phi(\tilde{y}|x)$ and $\max_\theta P_\theta(\tilde{y}|x,z)$. This function is continuous with respect to $x$ and $z$. We note that here we use $z$ but not $\tilde{z}$ because we will find out how to learn the $z$ to maximize the residual.
Also, in the following, we will keep using $y$ in replacement of $\tilde{y}$ to simplify the notations.
\begin{definition}\label{def:residual}
    Define $\epsilon(z)$, for $(x,y)\sim P_D$ and $z\in Z$, is the residual of $\log P_{\theta^*}(y|x,z) - \log P_{\phi^*}(y|x)$, where $\theta^*$ is the optimal parameters of $P(y|x,z)$ and $\phi^*$ is the optimal parameters of $P(y|x)$. The lower bound of $\epsilon(z)$ can be derived as follows.
    \begin{equation}
        \epsilon(z) \geq  \log \frac{P_{\phi^*}(x,y,z)}{P_D(x,y) P_D(z)}
    \end{equation}
\end{definition}
The derivation is as follows. Note that the prior distributions $P(x)$, $P(y)$, $P(z)$ are available from the dataset. Therefore we label them with subscript $D$ and we do not need to model these priors.
In addition, we assume that $P(x)$ and $P(z)$ are identically independent distributed, i.e., $P(x,z) = P(x)P(z)$, which holds in some real-world cases and also satisfies the definition in Granger causality~\cite{granger2004time}, which indicates that $z$ should contain some unique information that does not exist in other variables.

\begin{equation}
    \begin{split}
        \epsilon(z) & \coloneqq \log P_{\theta^*}(y|x,z) - \log P_{\phi^*}(y|x) \\
        & \geq \log P_{\phi^*}(y|x,z) - \log P_{\phi^*}(y|x) \\
        & = \log \frac{P_{\phi^*}(x,y,z)}{P_{\phi^*}(x,z)} \frac{P_{\phi^*}(x)}{P_{\phi^*}(x,y)}\\
        & \geq \log \frac{P_{\phi^*}(x,y,z)P_{D}(x)}{P_{D}(x)P_{D}(z)P_{D}(x,y)}\\
        & = \log \frac{P_{\phi}(x,y,z)}{P_D(x,y) P_D(z)} \\
    \end{split}
\end{equation}

To find out the latent $z$ that contributes most to the generation of $y$, our aim is to sample a $\hat{z}\sim p_\sigma(z|x)$ which makes $p_{\phi^*}(y|x,\hat{z})$ as close as possible to $p_{\theta^*}(y|x,\tilde{z})$, that is to maximize the residual $\epsilon(\hat{z})$.
\begin{theorem}
The target to find $\hat{z}$ that makes $\log P_{\phi}(y|x,\hat{z})$ as close as possible to the upper bound $\log P_{\theta^*}(y|x,\tilde{z})$ is formulated as:
    \begin{equation}
        \min_\sigma \log P_{\theta^*}(y|x,\tilde{z}) - \log P_{\phi}(y|x,\hat{z})
    \end{equation}
    The tighter lower bound is:
    \begin{equation}\label{eq:optimization}
        \arg\max_{\sigma} \log P_\phi(x,y,\hat{z}) \text{, subject to } \min_{\sigma} I(x;\hat{z})
    \end{equation}
    Note that we suppose a $\phi$ that is able to utilize $z$.
\end{theorem}

\begin{proof}
    because $\tilde{z}$ is the optimal $z$ and $\hat{z}$ is an arbitrary $z$, the inequalities hold
    \begin{equation}
        p_{\phi}(y|x) < p_{\phi}(y|x,\hat{z}) \leq p_{\theta^*}(y|x,\tilde{z})
    \end{equation}
    
    Therefore, to optimize the model with respect to $\hat{z}$ is the same as to minimize the difference between residuals $\epsilon(\hat{z})$ and $\epsilon(\tilde{z})$, since $\hat{z}$ is based on $\phi$ and does not change the lower bound of $\epsilon$ in Definition~\ref{def:residual}.
    
    \begin{equation}
        \begin{split}
            & \arg\min_\sigma \log P_{\theta^*}(y|x,\tilde{z}) - \log P_{\phi}(y|x,\hat{z})\\
            & = \arg\min_{\sigma} \epsilon(\tilde{z}) - \epsilon(\hat{z})\\
            & = \arg\max_\sigma \epsilon(\hat{z})\\
        \end{split}
    \end{equation}
    We are not able to compute the gradient of $\sigma$ of maximizing $\epsilon(\hat{z})$ since the process $\hat{z}\sim P_\sigma(z|x)$ is not differentiable. Therefore, we optimize its lower bound, at least to give the $\epsilon(\hat{z})$ a tighter constraint.
    \begin{equation}
        \begin{split}
            & \arg\max_\sigma \log P_\phi(x,y,\hat{z})\\
            & \text{subject to } P_\phi(x,\hat{z}) = P_D(x)P_D(\hat{z})
        \end{split}
    \end{equation}
    Because logarithm an increasing function, we choose to optimize its variable:
    \begin{equation}\small
        \nabla_\sigma P_\phi(x,y,\hat{z}) \approx \frac{1}{m} \sum_{i=1}^m p_\phi(y^i|x^i,\hat{z}^i)\nabla_\sigma \log p_\sigma(\hat{z}^i|x^i)
    \end{equation}
    where $m$ is the batch size.
    
    The constraint $P_\phi(x,z) = P_D(x)P_D(z)$ can be rewritten as minimizing the mutual information of x and z variables $I(X;Z) = E_{p(x,z)} \log \frac{p(x,z)}{p(x)p(z)}=E_{p(x)}D_{KL}(p(z|x)||p(z))$.
\end{proof}

\section{Computational Costs}
\begin{table}[h]\small
    \centering
    \begin{tabular}{lcc}\toprule[1pt]
        \bf Method & Training (hrs)\\\midrule[0.5pt]
        KGround & $\sim$6\\
        MLE & $\sim$5\\
        CVAE & $\sim$8\\
        InjK & $\sim$13\\\bottomrule[1pt]
    \end{tabular}
    \caption{The required training times of the methods on Personachat.}
    \label{tab:training-time}
\end{table}

Since CVAE samples one more token and InjK samples a sequence of tokens for computing KL divergence, CVAE consumes higher computational costs than MLE and InjK requires even more. To reduce this training cost is a worth discussing future work.

\section{Validation Results}
\begin{table}[h]\small
    \centering
    \begin{tabular}{lcc}\toprule[1pt]
        \bf Method & Hits@1 & PPL \\\midrule[0.5pt]
        \multicolumn{3}{l}{\it (Persona-Chat, w/ paired knowledge)}\\
        KGround$^*$ & 78.96 & 15.34\\
        MLE & 63.81 & 18.88\\
        CVAE & 65.15 & 18.20\\
        InjK & 68.12 & 17.94\\\midrule[0.5pt]
        \multicolumn{3}{l}{\it (LIGHT, w/ paired knowledge)}\\
        KGround$^*$ & 69.31 & 22.05\\
        MLE & 54.16 & 23.67\\
        CVAE & 54.91 & 23.39\\
        InjK & 60.25 & 22.92\\\midrule[0.5pt]
    \end{tabular}
    \caption{The results on validation set.}
    \label{tab:valid}
\end{table}

\end{document}